\documentclass[final,5p,times,twocolumn]{elsarticle}
\usepackage[T1]{fontenc}
\usepackage{amsmath,amssymb,amsthm}
\newtheorem{assumption}{Assumption}
\newtheorem{proposition}{Proposition}

\newtheorem{corollary}{Corollary}
\usepackage{graphicx}
\usepackage{booktabs}
\usepackage{url}
\usepackage{placeins}
\graphicspath{{./}}
\journal{arXiv preprint}
\myfooter[L]{Preprint. Submitted to Journal of the Franklin Institute.} 

\begin{document}
\begin{frontmatter}

\title{Quantile-Based Command-Reserve MPC for Current-Command Gimbal Pointing}

\author[sasee]{Yaokun Lu}
\author[sasee]{Yifan Li}
\author[hzii]{Enmei Wang\corref{cor1}}
\author[sisoe]{Hao Teng\corref{cor1}}
\author[hzii]{Pengwei Hu}
\author[sasee]{Jianzhong Qiao}
\cortext[cor1]{Corresponding authors. E-mail addresses: enmeiwang@126.com (E. Wang), hteng@buaa.edu.cn (H. Teng).}

\affiliation[sasee]{organization={School of Automation Science and Electrical Engineering, Beihang University},
  city={Beijing}, postcode={100191}, country={China}}
\affiliation[hzii]{organization={Hangzhou Innovation Institute of Beihang University},
  city={Hangzhou}, postcode={311115}, country={China}}
\affiliation[sisoe]{organization={School of Instrumentation Science and Optoelectronic Engineering, Beihang University},
  city={Beijing}, postcode={100191}, country={China}}
\begin{abstract}
This paper studies the constrained line-of-sight pointing control problem of a
two-axis current-command gimbal subjected to mismatched mechanical disturbances.
Friction, torque ripple, current-response lag, model mismatch, and limited drive
commands make accurate pointing more challenging yet interesting. A quantile
disturbance predictor is proposed to compensate predictable acceleration
residuals and describe their remaining uncertainty, with the aid of learning
median and outer quantiles from state, command, and reference histories.
A tube-inspired model predictive controller is proposed to reserve command
headroom according to prediction uncertainty, with the aid of calibrated
radii and stage-dependent input margins. The median improves the nominal
prediction, while the radius adapts the command reserve without adding
optimization variables. A finite-horizon tube inclusion result is established
for a specified linear error model and fixed ancillary policy; its conditions
are distinguished from the approximations used in hardware. Native command
bounds and a final command limiter keep the sent commands within their
prescribed range. On a two-axis embedded platform, four controller
configurations are compared in seven repeated trials each. The proposed
configuration achieves a two-dimensional pointing root-mean-square error of
$3.94\pm0.10$\,mrad with no command-limit samples in the evaluation windows.
Its performance supports the combined use of residual prediction and adaptive
command margins in the tested operating condition. Solver times and logged
loop overruns are reported separately for the nominal $100$\,Hz implementation.
\end{abstract}


\end{frontmatter}

\section{Introduction}
Current-command gimbals perform line-of-sight pointing in imaging, tracking,
and stabilization systems~\cite{kennedy2003direct,chung2025enhanced,rowienicz2026stabilization}.
Their drives expose a torque-producing current command to the outer pointing
controller. The measured current follows this command through an inner loop,
and the two signals differ during transients. Accurate pointing must therefore
account for the current response while respecting limited command authority.
Friction, torque ripple, and cross-axis effects further degrade the prediction
of the mechanical motion~\cite{armstrong1994survey,cui2024analysis}.

Disturbance rejection has been studied extensively for gimbals. Relay control,
for example, exploits limit-cycle behavior to attenuate exogenous torque
disturbances~\cite{yoon2022relay}. Observer-based approaches estimate composite
uncertainties for compensation~\cite{cui2024analysis,wang2024disturbance}.
For motor servo systems, additive state decomposition separates predictive
tracking from observer-based disturbance rejection~\cite{jiang2023decomposition}.
Neural-network MPC combined with active disturbance rejection has also been
validated on a robot manipulator~\cite{aouaichia2025neural}.
These studies motivate the joint treatment of model correction and constrained
tracking. A point prediction, however, does not by itself quantify the residual
uncertainty used to choose a command reserve.

Tube MPC propagates bounded model errors around a nominal trajectory and
tightens constraints accordingly~\cite{langson2004robust,mayne2005robust,mayne2014model}.
Inverse-dynamics and disturbance-observer-based tube control combines
disturbance compensation with a remaining-error tube for constrained
manipulator tracking~\cite{yang2023observer}. Composite anti-disturbance tube MPC
also uses predicted disturbances in the optimization~\cite{xie2026composite}.
Stochastic MPC explicitly links chance-constraint tightening with separate
feasibility and stability conditions~\cite{lorenzen2017constraint}.
Cautious MPC uses Gaussian-process residual models and propagates their
uncertainty within a chance-constrained formulation~\cite{hewing2020cautious}.

Adaptive or learned tubes are already established directions: adjustable
zonotope tubes have been combined with neural compensation for vehicle
tracking~\cite{wu2024adjustable}; deep quantile regression has been used to
learn trajectory tubes~\cite{fan2020tubes}; and learning-based rigid tube MPC
updates uncertainty descriptions from data~\cite{gao2024learning}.
Other tube constructions optimize polytopic shapes under configuration
constraints~\cite{villanueva2024configuration}, use time-varying cross-sections
for bounded-disturbance systems~\cite{wang2025timevarying}, or incorporate
state-dependent disturbance models into dynamic tubes~\cite{surma2025statedependent}.
These approaches treat tube geometry as part of the robust-control design;
the present hardware implementation uses a calibrated residual radius to
choose command reserves and does not inherit those guarantees.
Learning error bounds can also be connected to robust constraint tightening
in nonlinear economic MPC~\cite{xiong2025learning}. More generally, learning-based
MPC and predictive safety filters distinguish model improvement from the
explicit treatment of constraint protection~\cite{hewing2020learning,wabersich2021predictive}.
Thus, the use of learning or a varying tube is not by itself a new principle.
The issue considered here is how to use a shared residual predictor for both
nominal correction and a changing command reserve at a gimbal's current-command
interface, with an explicit account of what the deployed implementation verifies.

We learn the median and outer quantiles of an acceleration-domain residual
from recent measurements and reference preview. The median corrects a
current-augmented mechanical predictor, while a calibrated radius determines
stage-dependent command margins. The analytical tube construction and its
fixed-policy conditions are stated separately from the practical input-margin
implementation. This distinction matters because a mechanical residual label
does not establish a bound on every error of the augmented current state.

The contributions are summarized as follows.
\begin{enumerate}
\item A shared quantile predictor is developed for acceleration residuals in
two-axis current-command pointing. Our predictor is capable of providing a
median correction and a calibrated error radius from one recurrent model.
These outputs connect nominal prediction accuracy and command reserve to
the same measurement history, while their coverage is assessed empirically.
\item A tube-inspired MPC implementation is introduced to use both outputs
without increasing the optimization dimensions. A conditional finite-horizon
inclusion result clarifies the corresponding tube principle, and the
implemented nonlinear predictor, command margins, and safety monitoring are
specified separately. Four configurations, with seven hardware trials each,
demonstrate the tracking and command-headroom performance of the combined
scheme in the tested condition.
\end{enumerate}

\section{Current-Command Gimbal Model and Problem Formulation}
\label{sec:model}
This section formulates the continuous current-command gimbal model and the
physical signal limits. The sampled prediction model and control architecture
are arranged in Section~\ref{sec:method}.

\subsection{Current-Command Gimbal Pointing Model}
The pointing dynamics of a two-axis current-command gimbal are modeled with
the joint angles
$\boldsymbol\theta=[\theta_{\mathrm{az}},\theta_{\mathrm{el}}]^{\mathrm T}$,
joint rates $\boldsymbol\omega=\dot{\boldsymbol\theta}$, and drive current
$\boldsymbol i=[i_{\mathrm{az}},i_{\mathrm{el}}]^{\mathrm T}$. The subscripts
$\mathrm{az}$ and $\mathrm{el}$ denote the azimuth and elevation axes,
respectively. The control command $\boldsymbol u$ is sent to the
torque-producing channel of the drive, whereas $\boldsymbol i$ is returned by
the inner current loop. A first-order current-response model
approximates the synchronous-frame current loop~\cite{kim2010discrete}, so
$\boldsymbol u$ and $\boldsymbol i$ are modeled as distinct signals.
Computation-induced delay can also degrade predictive current control
if omitted from the design~\cite{cortes2012delay}.

The combined current-command pointing model is
\begin{equation}
\left\{
\begin{aligned}
\dot{\boldsymbol\theta} &= \boldsymbol\omega,\\
\boldsymbol M(\boldsymbol\theta)\dot{\boldsymbol\omega}
\!&= -\boldsymbol C(\boldsymbol\theta,\boldsymbol\omega)\boldsymbol\omega
   -\boldsymbol D\boldsymbol\omega
   -\boldsymbol G(\boldsymbol\theta)
   +\boldsymbol K_{\mathrm t}\boldsymbol i+\boldsymbol d_{\mathrm m},\\
\boldsymbol T_{\mathrm c}\dot{\boldsymbol i}
\!&=-\boldsymbol i+\boldsymbol u+\boldsymbol d_{\mathrm c},
\end{aligned}
\right.
\label{eq:continuous_current_command}
\end{equation}
where all signals are in $\mathbb R^2$,
$\boldsymbol M(\boldsymbol\theta)$ is the inertia matrix,
$\boldsymbol C(\boldsymbol\theta,\boldsymbol\omega)$ captures Coriolis and
cross-axis coupling terms, $\boldsymbol D$ is the viscous damping matrix,
$\boldsymbol G(\boldsymbol\theta)$ is the gravity and imbalance torque vector,
and
$\boldsymbol K_{\mathrm t}=\mathrm{diag}(K_{\mathrm t,\mathrm{az}},
K_{\mathrm t,\mathrm{el}})$ is the torque-constant matrix. The time-constant
matrix of the current response is
$\boldsymbol T_{\mathrm c}=\mathrm{diag}(\tau_{\mathrm c,\mathrm{az}},
\tau_{\mathrm c,\mathrm{el}})$. The load-side torque disturbance
$\boldsymbol d_{\mathrm m}$ represents friction, torque ripple, load and base
disturbances, parameter mismatch, and unmodeled cross-axis effects. The
current-channel residual $\boldsymbol d_{\mathrm c}$ represents the part not
captured by the first-order current-response model, including command latency,
PWM dead time, and voltage saturation.

Retaining the current layer in \eqref{eq:continuous_current_command} brings
three modeling difficulties. First, the mechanical torque follows the command
through the current response. Second, the mechanical disturbance enters the
rate dynamics, whereas the command enters the current dynamics; these channels
are mismatched. Third, the current-channel residual may affect both future
current and mechanical motion. These effects need not admit a common
two-dimensional disturbance representation in the full augmented state.

The physical signal limits are component-wise angle, rate, measured-current,
and command bounds. Let $\boldsymbol x=[\boldsymbol\theta^{\mathrm T},
\boldsymbol\omega^{\mathrm T},\boldsymbol i^{\mathrm T}]^{\mathrm T}$, and write
the associated state and input sets as
\begin{equation}
\begin{aligned}
\mathcal X=\{\boldsymbol x:\;&|\boldsymbol\theta|\leq\boldsymbol\theta_{\max},
\ |\boldsymbol\omega|\leq\boldsymbol\omega_{\max},
\ |\boldsymbol i|\leq\boldsymbol i_{\max}\},\\
\mathcal U=\{\boldsymbol u:\;&|\boldsymbol u|\leq\boldsymbol u_{\max}\}.
\end{aligned}
\label{eq:physical_sets}
\end{equation}
These sets define the ideal constrained-control problem. In the hardware
implementation, commands are bounded in the optimizer and before transmission;
angle and rate limits are checked by a runtime monitor. Predicted state-set
tightening is not implemented as a hard constraint. This practical distinction
is made explicit in Section~\ref{sec:deployment}.

The hardware acceleration baseline is an identified nonlinear function
$\boldsymbol f_{\mathrm{id}}(\boldsymbol\theta,\boldsymbol\omega,\boldsymbol i)$,
using currents, rates, and trigonometric angle terms. It approximates the
mechanical part of \eqref{eq:continuous_current_command}, rather than replacing
all unmodeled effects with an exact disturbance channel. A local linear
surrogate is used below only to state the conditional tube result; the
implemented predictor retains its nonlinear angle dependence.

\section{Quantile-Based Command-Reserve MPC}
\label{sec:method}
The control architecture in Fig.~\ref{fig:architecture} connects a residual
predictor to nominal model correction and uncertainty-dependent command
reserves. We first define the learned quantity, then state the analytical
tube result, and finally describe the deployed approximation.

\begin{figure*}[!t]
\centering
\includegraphics[width=0.98\textwidth]{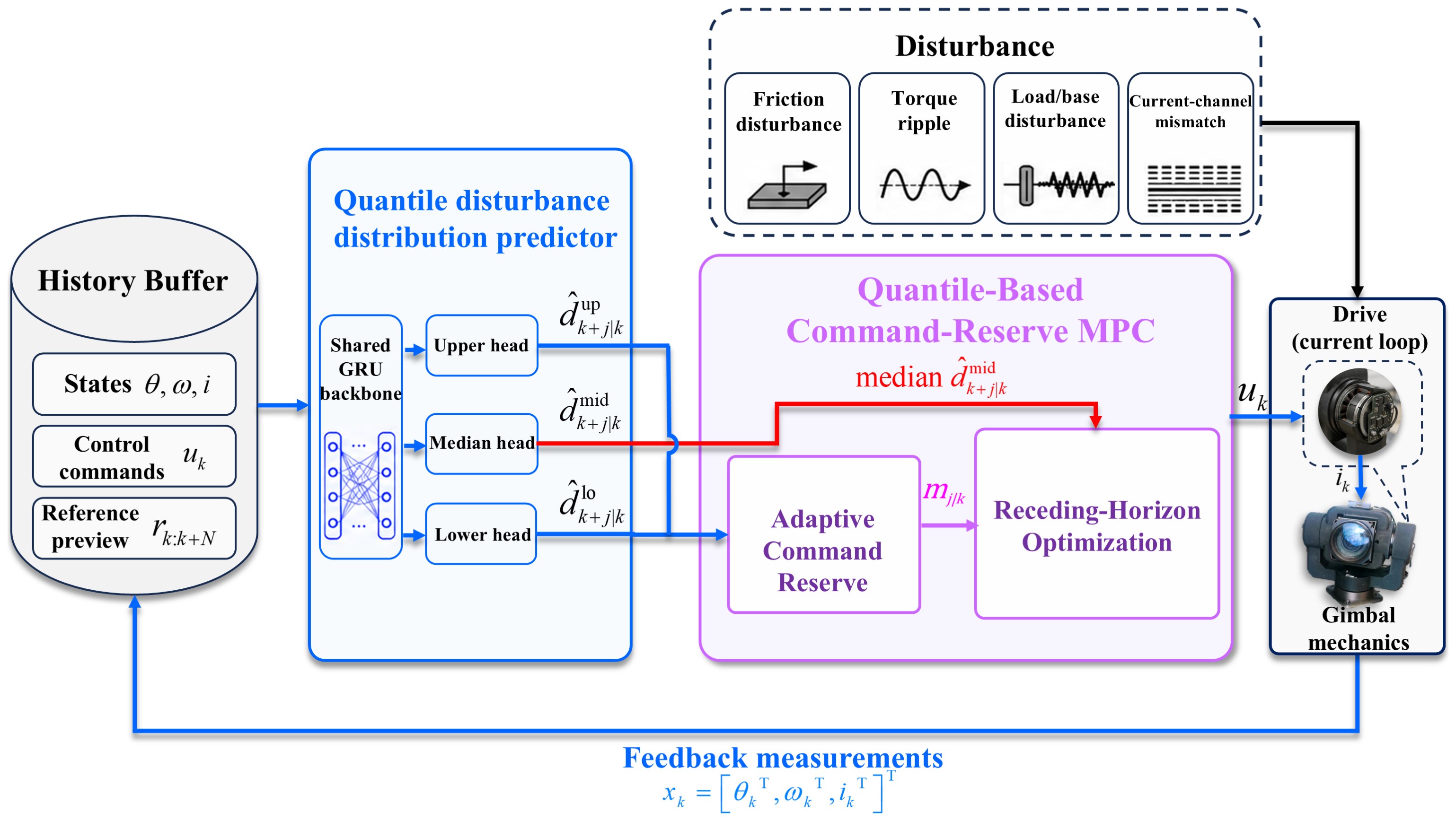}
\caption{Conceptual architecture for current-command gimbal pointing. The
predictor provides a median and outer quantiles. The adaptive command reserve
sets the input margin $m_{j|k}$. Analytical tube propagation is distinguished
from the hardware command-margin implementation
in Section~\ref{sec:deployment}. The disturbance sources shown are physical
sources affecting the plant; separate identification of each source is not
claimed.}
\label{fig:architecture}
\end{figure*}

\subsection{Quantile Residual Predictor and Calibration}
\label{sec:predictor}
For a logged sample $k$, the posterior acceleration-residual label is
\begin{equation}
\boldsymbol d_k^+=\mathcal D[\boldsymbol\omega]_k
-\boldsymbol f_{\mathrm{id}}(\boldsymbol\theta_k,
\boldsymbol\omega_k,\boldsymbol i_k),
\label{eq:disturbance_label}
\end{equation}
where $\mathcal D$ denotes numerical differentiation with the data-processing
filter, and $+$ marks a reconstructed label, not a forward-time index.
With an exact synchronized derivative and exact mechanical baseline, this
quantity reduces to $\boldsymbol M^{-1}\boldsymbol d_{\mathrm m}$.
With sampled and filtered measurements, it may also contain mechanical
effects of current variation within the sampling interval, discretization
error, identification error, and measurement noise. It is therefore an
effective mechanical prediction residual, not an independently measured
torque disturbance or a separately identified current-channel residual.

In particular, the presence of current-channel effects in a sampled
acceleration label does not prove that it bounds the independent future
current-state error. Such a claim would require checking the full sampled
state residual and its component outside the chosen disturbance channel.
We do not use label-prediction accuracy as evidence of full augmented-state
coverage.

A gated recurrent unit (GRU)~\cite{cho2014learning} processes a history of
$H$ feature vectors. These vectors contain state and past-command information,
tracking errors, and reference preview. Symbolically,
\begin{equation*}
\boldsymbol\xi_k=\{\boldsymbol x_{k-H+1:k},
\boldsymbol u_{k-H:k-1},\boldsymbol r_{k:k+N}\},
\end{equation*}
where $\boldsymbol r$ contains angle and rate references. The notation $k+j|k$
denotes a forecast for step $k+j$ made at $k$. A shared recurrent backbone and
output layer produce three quantiles for each axis and horizon step:
\begin{equation}
\{\hat{\boldsymbol d}^{\mathrm{lo}}_{k+j|k},
\hat{\boldsymbol d}^{\mathrm{mid}}_{k+j|k},
\hat{\boldsymbol d}^{\mathrm{up}}_{k+j|k}\}_{j=0}^{N-1}
=\mathcal P_{\Theta}(\boldsymbol\xi_k).
\label{eq:predictor}
\end{equation}
The target quantile levels are $\mathcal T=\{\alpha/2,0.5,1-\alpha/2\}$.
For scalar label component $a$ and training window $\ell$, define
$\varepsilon_{\ell,j,a}^{\tau}=d^+_{\ell+j,a}
-\hat d^\tau_{\ell+j|\ell,a}$. The quantile objective is
\begin{equation}
\mathcal L_Q(\Theta)=\sum_{\ell,j,a,\tau\in\mathcal T}
\max\{\tau\varepsilon_{\ell,j,a}^{\tau},
(\tau-1)\varepsilon_{\ell,j,a}^{\tau}\}.
\label{eq:pinball}
\end{equation}
This pinball loss targets conditional quantiles at the population optimum
\cite{koenker1978regression}; finite-data training does not ensure exact
conditional coverage.

Calibration uses a separate set of windows $\mathcal C$. For every axis $a$
and step $j$, let $c_{a,j}\geq0$ be the empirical $(1-\alpha)$-quantile of
\begin{equation*}
\max\{\hat d^{\mathrm{lo}}_{\ell+j|\ell,a}-d^+_{\ell+j,a},
d^+_{\ell+j,a}-\hat d^{\mathrm{up}}_{\ell+j|\ell,a},0\},
\quad \ell\in\mathcal C.
\end{equation*}
The implemented symmetric radius is
\begin{equation}
\begin{aligned}
\rho_{a,j|k}=\max\bigl\{0,\;&\max\{
\hat d^{\mathrm{mid}}_{k+j|k,a}-\hat d^{\mathrm{lo}}_{k+j|k,a},\\[-2pt]
&\hat d^{\mathrm{up}}_{k+j|k,a}-\hat d^{\mathrm{mid}}_{k+j|k,a}\}
+c_{a,j}\bigr\}.
\end{aligned}
\label{eq:error_set}
\end{equation}
For ordered quantiles this symmetric band contains the expanded outer-quantile
interval. The outer maximum prevents a negative radius if heads cross.
Standard split-conformal coverage guarantees rely on
exchangeability~\cite{angelopoulos2023conformal}. Adaptive conformal inference
addresses distribution shift through online calibration
adjustment~\cite{gibbs2021adaptive}. Here, the offsets remain fixed after
calibration. Although related to interval calibration
methods~\cite{romano2019conformalized,xu2023sequential}, this empirical
procedure is not a proof of joint-horizon coverage for dependent closed-loop
data. Nominal $1-\alpha$ refers to the target for a scalar axis-step
pair, not to the joint event over two axes and $N$ steps.

\subsection{Conditional Finite-Horizon Tube Construction}
\label{sec:theory}
To describe the tube principle, consider the following linear surrogate in
deviation coordinates:
\begin{equation}
\boldsymbol x_{k+1}=\boldsymbol A\boldsymbol x_k
+\boldsymbol B\boldsymbol u_k+\boldsymbol E\boldsymbol d_k.
\label{eq:discrete_model}
\end{equation}
Here $\boldsymbol d_k\in\mathbb R^2$ represents a mechanical-acceleration
disturbance and $\boldsymbol E$ its sampled-state channel. Equation
\eqref{eq:discrete_model} is an assumption on the error model used for the
following result. An independent current-state residual, sensor error, or
linearization remainder outside this channel would require an additional
bounded term. The learned labels in \eqref{eq:disturbance_label} alone do not
establish that this assumption holds for the physical six-state plant.

Fix a prediction time $k$, a median sequence, and a nominal plan
$(\boldsymbol z_{j|k},\boldsymbol\nu_{j|k})$ satisfying
\begin{equation}
\boldsymbol z_{j+1|k}=\boldsymbol A\boldsymbol z_{j|k}
+\boldsymbol B\boldsymbol\nu_{j|k}
+\boldsymbol E\hat{\boldsymbol d}^{\mathrm{mid}}_{k+j|k},
\quad \boldsymbol z_{0|k}=\boldsymbol x_k.
\label{eq:nominal_model}
\end{equation}
Define $\boldsymbol x^{(k)}_j$ as the trajectory obtained by retaining this
plan and applying its ancillary policy
\begin{equation*}
\boldsymbol u^{(k)}_j=\boldsymbol\nu_{j|k}
+\boldsymbol K(\boldsymbol x^{(k)}_j-\boldsymbol z_{j|k}),
\qquad \boldsymbol x^{(k)}_0=\boldsymbol x_k.
\end{equation*}
With $\boldsymbol A_K=\boldsymbol A+\boldsymbol B\boldsymbol K$,
$\boldsymbol e^{(k)}_j=\boldsymbol x^{(k)}_j-\boldsymbol z_{j|k}$, and
$\tilde{\boldsymbol d}_{j|k}=\boldsymbol d_{k+j}
-\hat{\boldsymbol d}^{\mathrm{mid}}_{k+j|k}$, this fixed-policy error obeys
\begin{equation}
\boldsymbol e^{(k)}_{j+1}=\boldsymbol A_K\boldsymbol e^{(k)}_j
+\boldsymbol E\tilde{\boldsymbol d}_{j|k}.
\label{eq:error_dynamics}
\end{equation}

\begin{assumption}[Horizon disturbance-envelope event]
\label{assum:env}
For the fixed prediction at $k$, the realized disturbance errors satisfy
$|\tilde{\boldsymbol d}_{j|k}|\leq\boldsymbol\rho_{j|k}$ component-wise
for $j=0,\ldots,N-1$.
\end{assumption}

In this result, the disturbance sequence is evaluated along the fixed-policy trajectory. This assumption describes a joint event. Its probability is not established
by training the individual quantiles. Propagate a component-wise box radius by
\begin{equation}
\boldsymbol s_{0|k}=\boldsymbol0,\qquad
\boldsymbol s_{j+1|k}=|\boldsymbol A_K|\boldsymbol s_{j|k}
+|\boldsymbol E|\boldsymbol\rho_{j|k}.
\label{eq:tube}
\end{equation}
The corresponding nominal constraints are
\begin{equation}
\begin{aligned}
\boldsymbol z_{j|k}&\in\mathcal X\ominus
[-\boldsymbol s_{j|k},\boldsymbol s_{j|k}],&&j=0,\ldots,N,\\
\boldsymbol\nu_{j|k}&\in\mathcal U\ominus
[-|\boldsymbol K|\boldsymbol s_{j|k},|\boldsymbol K|\boldsymbol s_{j|k}],
&&j=0,\ldots,N-1.
\end{aligned}
\label{eq:tighten}
\end{equation}
Here $\ominus$ denotes the Pontryagin difference and matrix absolute values
are element-wise.

\begin{proposition}[Fixed-policy finite-horizon inclusion]
\label{thm:tube}
Suppose \eqref{eq:discrete_model} is the complete error model for the trajectory
under consideration. Under Assumption~\ref{assum:env}, the fixed ancillary
policy above satisfies $|\boldsymbol e^{(k)}_j|\leq\boldsymbol s_{j|k}$
for $j=0,\ldots,N$. If its nominal plan satisfies \eqref{eq:tighten}, then
$\boldsymbol x^{(k)}_j\in\mathcal X$ and
$\boldsymbol u^{(k)}_j\in\mathcal U$ on their respective horizons.
\end{proposition}
\begin{proof}
The initial error is zero. If $|\boldsymbol e^{(k)}_j|\leq\boldsymbol s_{j|k}$,
then \eqref{eq:error_dynamics} and the envelope event yield
\begin{equation*}
|\boldsymbol e^{(k)}_{j+1}|\leq|\boldsymbol A_K|\boldsymbol s_{j|k}
+|\boldsymbol E|\boldsymbol\rho_{j|k}=\boldsymbol s_{j+1|k}.
\end{equation*}
Induction proves inclusion. The definition of the tightened sets then gives
the state and input memberships.
\end{proof}

No contraction condition is needed for this finite-horizon induction. An LQR
gain may be used to shape the tube, but Schur stability of $\boldsymbol A_K$
does not imply $\|\boldsymbol A_K\|_\infty<1$ or stability of
$|\boldsymbol A_K|$. An invariant or infinite-horizon tube requires additional
conditions beyond the result above.

\begin{corollary}[Conditional first-step constraint satisfaction]
\label{cor:firststep}
At a given solve time $k$, suppose the model in \eqref{eq:discrete_model} is
complete, the state is measured exactly, and a feasible nominal plan satisfies
\eqref{eq:tighten}. If its first command is applied without modification and
$|\boldsymbol d_k-\hat{\boldsymbol d}^{\mathrm{mid}}_{k|k}|
\leq\boldsymbol\rho_{0|k}$, then $\boldsymbol u_k\in\mathcal U$ and
$\boldsymbol x_{k+1}\in\mathcal X$.
\end{corollary}
\begin{proof}
Since $\boldsymbol s_{0|k}=0$, the applied command equals
$\boldsymbol\nu_{0|k}$. The next-state error is
$\boldsymbol E(\boldsymbol d_k-\hat{\boldsymbol d}^{\mathrm{mid}}_{k|k})$,
whose component-wise magnitude is bounded by $\boldsymbol s_{1|k}$.
The tightened sets give the result.
\end{proof}

Reoptimization changes the future nominal plan, so Proposition~\ref{thm:tube}
does not enclose the entire subsequent reoptimized trajectory around the old
plan. Reinitialization makes the new nominal initial error zero; it neither
changes the physical state nor proves feasibility of the next optimization.
Neither recursive feasibility nor closed-loop stability is claimed here.

\subsection{Nonlinear Prediction and Command-Margin Implementation}
\label{sec:deployment}
The hardware optimizer retains the current state and the nonlinear identified
acceleration. Using physical current commands $\boldsymbol\nu$ in amperes,
the discrete predictor is
\begin{equation}
\left\{\begin{aligned}
\boldsymbol i_{j+1|k}&=\boldsymbol i_{j|k}
+T_{\mathrm s}\boldsymbol T_{\mathrm c}^{-1}
(\boldsymbol\nu_{j|k}-\boldsymbol i_{j|k}),\\
\boldsymbol\omega_{j+1|k}&=\boldsymbol\omega_{j|k}
+T_{\mathrm s}\bigl(\boldsymbol f_{\mathrm{id}}(\boldsymbol\theta_{j|k},
\boldsymbol\omega_{j|k},\boldsymbol i_{j|k})+\boldsymbol\mu_{j|k}\bigr),\\
\boldsymbol\theta_{j+1|k}&=\boldsymbol\theta_{j|k}
+T_{\mathrm s}\boldsymbol\omega_{j+1|k}.
\end{aligned}\right.
\label{eq:nonlinear_model}
\end{equation}
The median correction $\boldsymbol\mu$ includes the fixed deployment scaling,
offset, and temporal filtering described below. The current response is a
nominal model, with no separate learned $\boldsymbol d_{\mathrm c}$ channel.
The software orders states axis-by-axis as
$(\theta_{\mathrm{az}},\omega_{\mathrm{az}},i_{\mathrm{az}},
\theta_{\mathrm{el}},\omega_{\mathrm{el}},i_{\mathrm{el}})$; this is a
permutation of the grouped notation used in Section~\ref{sec:model}.

The optimized objective is
\begin{align}
J_k=&\sum_{j=0}^{N-1}\bigl(
\|\boldsymbol C_y\boldsymbol z_{j|k}-\boldsymbol r_{k+j}\|_{\boldsymbol Q}^2
+\|\boldsymbol\nu_{j|k}-\boldsymbol b_k\|_{\boldsymbol R}^2\nonumber\\
&\quad+\|\boldsymbol\nu_{j|k}-\boldsymbol u^{\mathrm{sent}}_{k-1}
\|_{\boldsymbol S}^2\bigr)
+\|\boldsymbol C_y\boldsymbol z_{N|k}-\boldsymbol r_{k+N}
\|_{\boldsymbol P}^2,
\label{eq:mpc_obj}
\end{align}
subject to \eqref{eq:nonlinear_model}, the measured initial state, and the
native input bounds specified below. $\boldsymbol C_y$ selects angle and
rate, $\boldsymbol b_k$ is the live command bias, and all weights are
nonnegative. The effort penalty is centered on the bias. At every horizon
stage the command-change penalty uses the same previously sent command
$\boldsymbol u^{\mathrm{sent}}_{k-1}$, rather than the previous optimized
stage command. No measured-current stage penalty is included.

For axis $a$, the practical input margin in command-count coordinates is
\begin{equation}
m_{a,j|k}=g_{\rho}|\bar\rho_{a,j|k}|+m^K_{a,j|k},\qquad
L_{a,j|k}=\max\{L_{a,\min},L_{a,\max}-m_{a,j|k}\},
\label{eq:command_margin}
\end{equation}
and $|\nu^{\mathrm{count}}_{a,j|k}|\leq L_{a,j|k}$.
One command count corresponds to $0.01$\,A. The radius $\bar\rho$ is the
scaled calibrated radius. The implementation provides an optional
$m^K=\gamma\sum_l|K_l|\bar s_l$ term. Its per-axis auxiliary widths accumulate
acceleration uncertainty according to
\begin{equation*}
\bar s_{\omega,j+1}=\min\{s_{\omega,\max},
\bar s_{\omega,j}+T_{\mathrm s}|\bar\rho_j|\},
\end{equation*}
\begin{equation*}
\bar s_{\theta,j+1}=\min\{s_{\theta,\max},
\bar s_{\theta,j}+T_{\mathrm s}\bar s_{\omega,j+1}\}.
\end{equation*}
These capped accumulators, and a minimum allowed input range, are engineering
choices. They are not asserted to be the full box propagation
\eqref{eq:tube} or its guaranteed tightening \eqref{eq:tighten}. In particular,
the input margin is not proved to dominate $|\boldsymbol K|\boldsymbol s$.
Consequently, Proposition~\ref{thm:tube} and Corollary~\ref{cor:firststep} are
conditional analytical results, not a certified state-constraint guarantee for
this hardware controller.

The dynamics in \eqref{eq:nonlinear_model} depend nonlinearly on angle. The
original optimization is therefore a nonlinear optimal-control problem. The
acados implementation~\cite{verschueren2022acados} uses sequential quadratic
programming with real-time iteration (SQP-RTI)~\cite{diehl2005realtime},
a Gauss--Newton Hessian approximation, and a partially condensed HPIPM
QP solver~\cite{frison2020hpipm}. Updating the predicted median and radius does not
add states, controls, or constraints, so the QP subproblem dimensions remain
fixed. Its coefficients and solution time can still vary.

After optimization, the first command is blended with the operating bias,
passed through the configured slew limit, rounded, and limited to the outer
command bounds before CAN transmission. Thus the actual sent input need not
equal the raw optimizer output. Angle and rate monitoring can trigger a stop;
it does not replace predictive state-set tightening. The hardware results
below evaluate this deployed pipeline as a whole.

\begin{figure*}[!t]
\centering
\includegraphics[width=0.78\textwidth]{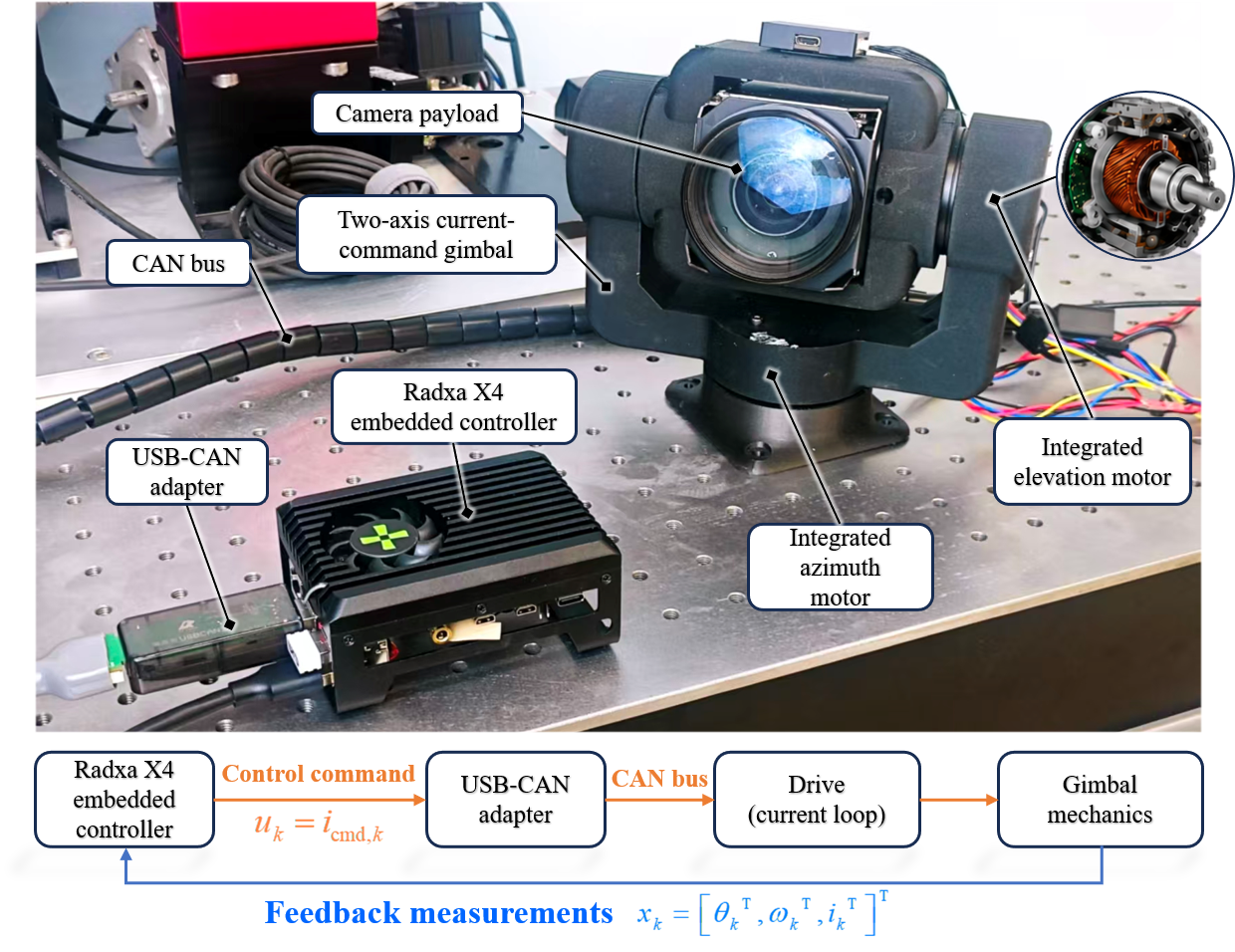}
\caption{Embedded two-axis current-command gimbal platform. A Radxa X4
controller exchanges commands and measurements over CAN. The inset shows the
integrated elevation drive.}
\label{fig:hardware_platform}
\end{figure*}

\section{Experimental Validation}
\label{sec:experiments}
The experiments evaluate the deployed controller's residual prediction,
pointing accuracy, command-boundary behavior, and timing. They do not test a
full augmented-state tube certificate.

\begin{table*}[!t]
\caption{Hardware comparison: mean $\pm$ sample standard deviation over seven
trials per configuration; statistics use $t\geq5$\,s in each $18$\,s run.}
\label{tab:hw}
\centering\small
\begin{tabular}{@{}lcccc@{}}
\toprule
Method & Azimuth RMSE (mrad) & Elevation RMSE (mrad) & 2D RMSE (mrad) & Command-limit samples\\
\midrule
MPC & $4.60\pm0.17$ & $5.88\pm1.18$ & $7.49\pm1.03$ & $3.9\pm2.7$\\
MPC-P & $2.92\pm0.16$ & $3.54\pm0.39$ & $4.59\pm0.35$ & $1.3\pm1.3$\\
Fixed Tube-MPC & $4.55\pm0.06$ & $4.06\pm0.63$ & $6.11\pm0.43$ & $0.0\pm0.0$\\
Proposed & $2.86\pm0.06$ & $2.70\pm0.20$ & $3.94\pm0.10$ & $0.0\pm0.0$\\
\bottomrule
\end{tabular}
\end{table*}

\subsection{Platform, Predictor, and Comparison Protocol}
The controller runs in C++/SocketCAN on the platform in
Fig.~\ref{fig:hardware_platform}, with a nominal sampling period
$T_{\mathrm s}=10$\,ms. Encoder angles, gyroscope rates, drive-current
feedback, command values, GRU outputs, and solver timing are logged. The
comparison comprises four configurations:
\begin{enumerate}
\item MPC: neither learned median correction nor radius-based input margin;
\item MPC-P: learned median correction without the radius-based margin;
\item Fixed Tube-MPC: fixed-radius input margin without median correction;
\item Proposed: median correction and prediction-dependent input margin.
\end{enumerate}
The name Fixed Tube-MPC denotes the fixed-margin hardware comparator; it
does not imply a certified invariant-tube implementation. Shared plant,
reference, and non-component controller settings are retained in the
comparison. Each configuration is run seven times for approximately $18$\,s.
The first $5$\,s are excluded from tracking and command-boundary statistics,
leaving an evaluation window of approximately $13$\,s per trial. Trials are
summarized individually before their means and sample standard deviations
are computed; time samples are not treated as independent repeated trials.

Over the evaluation window, the reference spans approximately $26.5^\circ$
in azimuth and $17.3^\circ$ in elevation peak-to-peak. The outer command limits
are $\pm0.70$\,A and $\pm1.15$\,A, corresponding to $70$ and $115$ counts.
The runtime angle-monitor thresholds are $\pm18^\circ$ and $\pm14^\circ$.
These outer bounds are distinct from the stage-dependent reserves in
\eqref{eq:command_margin}.

The GRU uses $H=48$ history samples, $N=15$ forecasts, $96$ hidden units, and
$51$ features. The shared output network consists of layer normalization,
a $96$-unit SiLU layer, and $90$ scalar outputs ($15\times2\times3$).
The feature set includes measured angles, rates, currents, previous commands,
tracking errors, IMU rates, trigonometric angles, and reference preview at
offsets $1$, $3$, $5$, and $10$ samples. The archived model-development manifest assigns complete recordings to
$13$ training, $4$ calibration, and $4$ test files before history windows
are constructed. The resulting sets contain $20\,243$, $6\,503$, and
$6\,491$ windows, respectively; overlapping windows within a recording
are not independent observations. The documented training procedure fits
the nominal acceleration model and feature normalization on the training
recordings. The calibration recordings are used both to select the network
checkpoint and to estimate the interval offsets. They are therefore a
validation/calibration set, rather than an untouched split-conformal
calibration set. The four offline test recordings are separate from these
training and calibration files. Quantile levels are $0.05$, $0.5$, and $0.95$.
The model-development report is dated June 5, 2026. The $28$ comparison
recordings collected on June 6, including trial 3 in
Fig.~\ref{fig:prediction}, are not among its listed development files;
the model-report hash recorded with these runs matches the archived report.
These later closed-loop trials are distinct from the offline test windows.
The identified acceleration basis comprises a constant, both measured currents,
both rates, and the sine and cosine of each joint angle, with separate
coefficients for each output axis.

The deployment filters the median sequence with a first-order temporal filter
and permits fixed residual offsets; these operations alter the center from
the raw quantile output. The archived trial metadata record unit residual
gains, a $-40$\,deg/s$^2$ elevation residual offset, command biases of $0$ and
$64$ counts, and first-command blend gains of $1$ and $0.95$ for azimuth and
elevation. These are fixed engineering deployment settings, rather than
learned quantiles. The archived metadata specify the same biases, blend
gains, and residual settings across all seven comparison batches. Biases
and blend gains are shared across the four configurations; the residual
offset is applied only when median correction is enabled (MPC-P and
Proposed). The comparison does not establish that these settings are
universally optimal or that controller performance is independent of tuning.
The radius is formed before median filtering. Coverage of the raw
network interval must therefore not be assumed for the shifted, filtered
control correction. The radius-dependent input reserves are evaluated through
the actual logged closed-loop behavior. The optional LQR-margin contribution
is zero in all $28$ trial logs; the tested reserve therefore comes from the
direct radius-to-command term in \eqref{eq:command_margin}.

\begin{figure}[!tb]
\centering
\includegraphics[width=0.98\columnwidth]{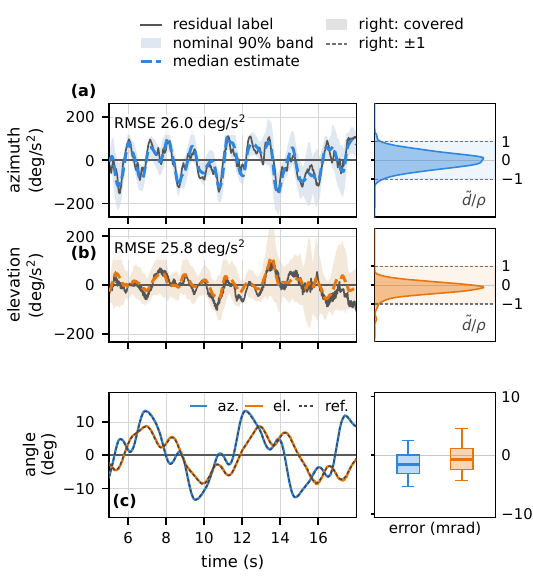}
\caption{Representative hardware trial. (a,b) Processed acceleration-residual
labels, logged median predictions, and calibrated symmetric prediction bands
for azimuth and elevation. The nominal scalar target is $90\%$; the reported
coverage is empirical. The right panels show residual prediction error
normalized by the radius. (c) Angle tracking and pointing-error distributions.}
\label{fig:prediction}
\end{figure}

\subsection{Prediction Accuracy and Coverage Interpretation}
Figure~\ref{fig:prediction} shows reconstructed acceleration residuals,
logged median predictions, and tracking in a representative Proposed trial.
This run is separate from the model-development files described above;
its processed residual scores are not the model report's offline test scores.
The plotted signals use a nine-sample moving-average display pipeline;
the rates are smoothed before numerical differentiation and the resulting
acceleration is smoothed again. A further centered nine-sample average is
applied to the reconstructed residual, logged median, and radius before
plotting and scoring. This is an offline evaluation, not a causal online
measurement of disturbance truth. In this processed evaluation, the residual
root-mean-square error (RMSE) decreases from $64.1$ to $26.0$\,deg/s$^2$ in azimuth and from $38.4$ to
$25.8$\,deg/s$^2$ in elevation. The baseline is the zero residual correction.
These figures describe agreement with reconstructed labels, rather than
directly measured physical disturbances.

To specify coverage, let $\mathcal I_r$ be the valid evaluation samples in
trial $r$. For an available label sequence, define the axis-step frequency
\begin{equation}
\widehat C_{a,j}^{(r)}=\frac{1}{|\mathcal I_r|}
\sum_{k\in\mathcal I_r}\mathbf1\{
|d^+_{k+j,a}-\hat d^{\mathrm{mid}}_{k+j|k,a}|\leq\rho_{a,j|k}\}.
\label{eq:empirical_coverage}
\end{equation}
Joint-axis or whole-horizon coverage would instead count the intersection
of the relevant events at each $k$; it cannot be inferred from their separate
frequencies. The $98.5\%$ and $97.9\%$ coverage values in
Fig.~\ref{fig:prediction} apply to $j=0$ for the azimuth and elevation
processed signals of that representative run (trial 3, $1\,299$ evaluation samples).
They do not measure joint
two-axis, $15$-step, or full-state coverage. No joint-horizon probability
guarantee is claimed. Panel~(c) shows tracking performance; predictability
is assessed by the residual comparisons in panels~(a,b).

\begin{figure}[!tb]
\centering
\includegraphics[width=0.98\columnwidth]{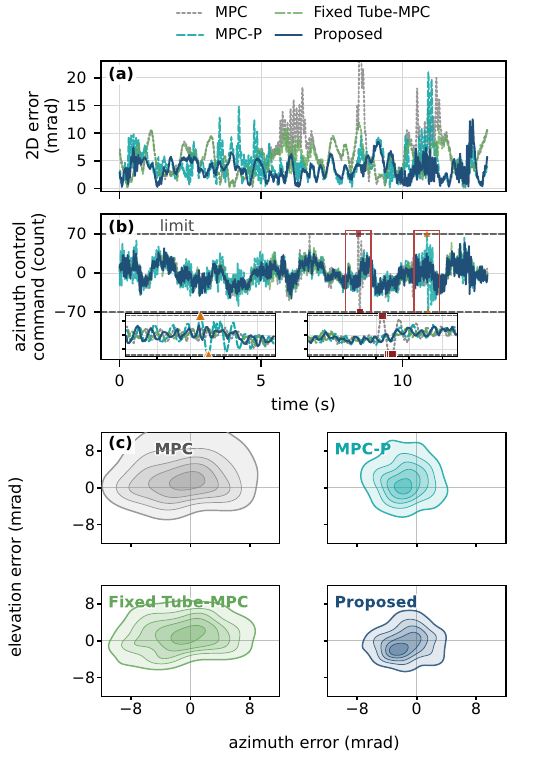}
\caption{Representative four-configuration hardware comparison under the
outer command limits. (a) Two-dimensional pointing error. (b) Azimuth command;
markers denote samples at the command boundary. (c) Pointing-error
distributions. This comparison includes conditional component comparisons,
but does not isolate adaptive versus fixed radii at a shared median correction.}
\label{fig:tracking}
\end{figure}

\subsection{Tracking and Command-Boundary Results}
For trial $r$, two-dimensional pointing RMSE is computed as
\begin{equation}
R_r=\sqrt{\frac{1}{|\mathcal I_r|}\sum_{k\in\mathcal I_r}
(e^2_{\mathrm{az},r,k}+e^2_{\mathrm{el},r,k})},
\label{eq:rmse}
\end{equation}
with angular errors converted to milliradians. Table~\ref{tab:hw} reports
$\overline R\pm s_R$, where $s_R$ is the sample standard deviation of the seven
trial RMSE values (denominator $6$). Axis-wise RMSE is calculated similarly.

The last column counts command-boundary samples per axis and sums the two
axis counts. A sample is counted when its saturation flag is set or the
absolute rounded logged command reaches the prescribed outer command limit.
This is a boundary-usage indicator, not a count of strict constraint
violations. A command equal to its limit still satisfies the non-strict
command constraint.

The Proposed configuration has the lowest observed mean 2D RMSE,
$3.94\pm0.10$\,mrad, with zero command-limit samples in the evaluation
windows. Its unrounded mean RMSE is $47.4\%$ below MPC, $14.2\%$ below MPC-P,
and $35.6\%$ below Fixed Tube-MPC. MPC-P versus MPC evaluates median correction
without a tube margin; Proposed versus MPC-P evaluates adding the implemented
adaptive margin when median correction is present. Proposed versus Fixed
Tube-MPC changes both median correction and the radius rule, so that
comparison cannot isolate the benefit of an adaptive radius relative to a
fixed radius. No claim of a separate causal advantage for radius adaptation
is made from this comparison. The results establish performance of the four
complete configurations under the tested reference and command limits.

\subsection{Timing and Scope of the Evidence}
No nonzero solver status is logged in the $28$ comparison trials. The maximum
logged MPC solve time across all configurations is $4.450$\,ms; the Proposed
maximum is $2.789$\,ms. Solve time is distinct from the active-loop duration,
which includes measurement acquisition, GRU prediction, optimizer calls, CAN
transmission, and post-send reads. The loop timer stops before writing the
CSV row and sleeping, so it is not a measurement of the complete scheduled
period or an end-to-end latency bound.

\begin{table}[!htb]
\caption{Logged timing for Proposed across seven trials. An overrun means
the measured active-loop work exceeds $10$\,ms.}
\label{tab:timing}
\centering\small
\begin{tabular}{@{}lr@{}}
\toprule
Quantity & Value\\
\midrule
Maximum MPC solve time & $2.789$\,ms\\
Maximum active-loop duration & $44.012$\,ms\\
Overruns, entire runs & $100/12\,588$ ($0.794\%$)\\
Overruns, $t\geq5$\,s & $70/9\,094$ ($0.770\%$)\\
\bottomrule
\end{tabular}
\end{table}

Table~\ref{tab:timing} shows that the implementation has occasional overruns.
Thus $100$\,Hz describes the nominal schedule, not a hard-real-time guarantee.
The fixed-period predictor also approximates intervals affected by scheduling
jitter. Solver success and small solver times do not remove this distinction.

The experiments support residual correction and command-headroom management
for the measured operating condition. They do not establish joint-horizon
coverage, a full-state tube bound, recursive feasibility, or robustness to
unmeasured disturbance regimes. In particular, safety monitoring and a final
command limiter are practical mechanisms, not substitutes for a proof of
predicted state-constraint satisfaction. A median-plus-fixed-radius comparator
would be needed to attribute improvements specifically to the adaptive radius.
These limitations do not alter the measured full-configuration comparison.

\section{Conclusion}
This paper presented a disturbance-distribution learned, tube-inspired MPC
for two-axis current-command pointing. A shared quantile predictor provides
a median acceleration-residual correction and a calibrated radius for
stage-dependent command reserves. A fixed-policy finite-horizon inclusion
result states the conditions behind analytical tube tightening; the hardware
uses an explicitly described nonlinear predictor and practical input margins.
In seven trials per configuration, the combined scheme attained
$3.94\pm0.10$\,mrad 2D pointing RMSE and no command-limit samples in the
evaluation windows. The comparison supports the complete controller under
the tested conditions, while timing records show occasional active-loop
overruns at a nominal $100$\,Hz schedule. Future work concerns explicit
current-channel uncertainty, joint-horizon calibration, and a verified
connection between the deployed input margins and recursive feasibility.

\section*{CRediT authorship contribution statement}
\textbf{Yaokun Lu:} Conceptualization, Methodology, Software, Investigation,
Writing -- original draft. \textbf{Yifan Li:} Investigation, Validation.
\textbf{Enmei Wang:} Supervision, Methodology, Writing -- review \& editing.
\textbf{Hao Teng:} Supervision, Resources, Writing -- review \& editing.
\textbf{Pengwei Hu:} Resources, Validation. \textbf{Jianzhong Qiao:} Funding
acquisition, Supervision, Project administration.

\section*{Declaration of competing interest}
The authors declare that they have no known competing financial interests or
personal relationships that could have appeared to influence the work reported
in this paper.

\section*{Declaration of generative AI and AI-assisted technologies in the manuscript preparation process}
During the preparation of this work, the authors used OpenAI Codex and
Anthropic Claude for implementation checks, editorial assistance, and
manuscript review, and DeepSeek, Qwen, and Google Gemini for additional
manuscript review and language editing. After using these tools, the authors
reviewed and edited the content as needed and take full responsibility for the
content of the published article.

\section*{Acknowledgments}
This work was supported in part by the National Natural Science Foundation of
China under Grants U24B20158, 62388101, and 62573034.

\section*{Data availability}
The hardware trial data underlying Table~\ref{tab:hw} are provided as
supplementary material.

\bibliographystyle{elsarticle-num}
\bibliography{references}

\end{document}